\documentclass[letterpaper, 10pt, conference]{ieeeconfarq}  

\IEEEoverridecommandlockouts                              

\usepackage[pdftex]{graphicx}
\usepackage{cite}
\usepackage{verbatim}		
\usepackage{amsmath}
\usepackage{amssymb}

\usepackage{amsthm}

\usepackage{mathtools}	
\usepackage{grffile}	
\usepackage[tight,footnotesize]{subfigure}
\usepackage{microtype} 
\usepackage{color}
\usepackage{url}
\usepackage[ruled,vlined,linesnumbered]{algorithm2e}
\usepackage{bm} 
\usepackage{comment}
\usepackage{arydshln}
\let\labelindent\relax
\usepackage{enumitem}

\usepackage{adjustbox}
\usepackage{tikz}
\usetikzlibrary{shadings, patterns, angles, quotes, arrows.meta, shapes, decorations.pathmorphing, decorations.shapes, decorations.text}
\usetikzlibrary{calc,intersections,arrows.meta}
\usepackage{pgfplots}

\usepackage{hyperref}
\hypersetup{
	colorlinks=true,
	linkcolor=blue,
	citecolor=red,
}

	\tikzset{
	pil/.style={
		->,
		thick,
		shorten <=2pt,
		shorten >=2pt,}
}

\newcommand{\cmmnt}[1]{}

\usepackage{amsthm}
\usepackage{amsfonts} 
\usepackage{amsmath}

\theoremstyle{plain}
\newtheorem{theorem}{Theorem}[section]

\newtheorem{corollary}[theorem]{Corollary}
\newtheorem{lemma}[theorem]{Lemma}
\newtheorem{prop}[theorem]{Proposition}
\theoremstyle{definition}

\newtheorem{definition}[theorem]{Definition}

\newtheorem{problema}[theorem]{Problem}

\theoremstyle{remark}
\newtheorem{rem}[theorem]{Remark}

\newcommand{\norm}[1]{\left| \left| {#1} \right| \right|}
\newcommand{\grad}{\vec{ \nabla}}

\title{\LARGE \bf
Source-Seeking Problem with Robot Swarms
}

\author{Antonio Acuaviva, Hector Garcia de Marina, Juan Jimenez 
	\thanks{This manuscript constitutes a summary of the first named author's (A. Acuaviva) Bachelor's dissertation, completed under the supervision of the second and third named authors (H. Garcia de Marina and J. Jimenez) at Universidad Complutense de Madrid.}}

\begin{document}

\maketitle
\thispagestyle{empty}
\pagestyle{empty}

\begin{abstract}
We present an algorithm to solve the problem of locating the source, or maxima, of a scalar field using a robot swarm. We demonstrate how the robot swarm determines its direction of movement to approach the source using only field intensity measurements taken by each robot. In contrast with the current literature, our algorithm accommodates a generic (non-degenerate) geometry for the swarm's formation. Additionally, we rigorously show the effectiveness of the algorithm even when the dynamics of the robots are complex, such as a unicycle with constant speed. Not requiring a strict geometry for the swarm significantly enhances its resilience. For example, this allows the swarm to change its size and formation in the presence of obstacles or other real-world factors, including the loss or addition of individuals to the swarm on the fly. For clarity, the article begins by presenting the algorithm for robots with free dynamics. In the second part, we demonstrate the algorithm's effectiveness even considering non-holonomic dynamics for the robots, using the vector field guidance paradigm. Finally, we verify and validate our algorithm with various numerical simulations.
\end{abstract}


\section{Introduction}

Source localization in scalar fields is currently considered one of the fundamental problems in swarm robotics. In particular, providing formal performance guarantees for an algorithm that is feasible to implement in practice is one of the major challenges in swarm robotics \cite{yang2018grand}. Effectively solving this problem would significantly impact how we monitor our environment \cite{ogren2004cooperative}, conduct search and rescue operations \cite{kumar2004robot}, and detect chemicals, sound sources, or pollutants \cite{marques2006particle, zhang2008maximum, li2006moth}. One of the key factors that make robot swarms ideal for such missions is their resilience; that is, the ability to maintain functionality even in the presence of adverse and unknown conditions. For instance, a robot swarm can continue operating and locate the source of the field even with the loss of (\emph{disposable}) robots at any time. Additionally, it possesses unlimited autonomy in time and energy, as robots can enter and leave the swarm as needed.

In this work, we propose a resilient algorithm for a robot swarm to solve the problem of source localization in a scalar field. This algorithm is practical for real-world applications for three key reasons: first, the robots only need to make point measurements of the field intensity at their positions; second, we allow for a generic non-degenerate deployment\footnote{For example, a degenerate deployment would be forming a line in the plane.} within the swarm that can vary arbitrarily; and third, we accommodate realistic restrictive robot dynamics, such as unicycles travelling at constant speed, similar to unmanned aerial vehicles \cite{yao2021singularity}.

Source localization in a field is associated with finding the maximum of a multivariable function, and various approaches to tackle this problem exist in the literature. One of the most common methods is the use of gradient descent techniques. If available, the signal gradient can be used to develop a gradient ascending/descent algorithm for a vehicle or a group of vehicles \cite{bachmayer2002vehicle}. However, in practice, robots can only measure the signal magnitude (scalar) and not the gradient (vector). Therefore, it is necessary to use the signal magnitudes in the space to estimate the gradient.

In the literature, it is common to require a specific spatial distribution for the robots in the swarm \cite{brinon2015distributed,brinon2019multirobot,brinon2019circular}, thus restricting their mobility and flexibility. Recently, the authors in \cite{al2021distributed} proposed an elegant control law where a specific spatial distribution for the robots is not required. However, the swarm's distribution is not controlled and could be considered \emph{chaotic} since it changes constantly over time depending on the initial positions of the robots. This can cause problems in certain scenarios, such as the presence of obstacles.

Our work proposes an alternative to gradient estimation by using an ascending direction to guide the process. By eliminating the need for gradient computation or estimation, the requirement for a specific geometry in robot deployment is removed. This approach allows for dynamic control over the deployment, enabling it to remain constant or adjust arbitrarily, thereby making the robot swarm more flexible and adaptable to real-world demands.

\section{The location problem}
\subsection{Preliminaries and Problem Formulation}

A signal can be modelled as a scalar field $\sigma(\vec{r})$, representing the signal's intensity at position $\vec{r} \in \mathbb{R}^n$, where $n = 2$ (plane) or $n = 3$ (space). Additionally, since we are examining signals generated by a source, this implies that there is a position $\vec{r}^*$, corresponding to the signal's source, which will be the maximum of our scalar field. Furthermore, additional smoothness conditions will be imposed, as outlined in the following definition.

\begin{definition}\label{def:señal}
    A \emph{signal distribution} is a scalar field $\sigma: \mathbb{R}^n \to \mathbb{R}^{+}$ that is continuous and twice differentiable, with globally bounded partial derivatives (up to second order). Additionally, we require that $\sigma$ has a unique maximum at $\vec{r}^*$, with $\nabla \sigma(\vec{r}) \neq 0$ for $\vec{r} \neq \vec{r}^*$, and $\lim_{\vec{r} \to \infty} \sigma(\vec{r}) = 0$.
\end{definition}

To model a swarm of $N$ robots, we can characterise each robot by its position $\vec{r}_i \in \mathbb{R}^n$, $i = 1, \dots, N$. Additionally, each robot will be equipped with a sensor that allows us to measure the signal, providing us with access to the field information at that point, $\sigma(\vec{r}_i)$.

We characterise a \emph{swarm of $N$ robots} by stacking their positions $\vec{r}_i \in \mathbb{R}^n$, for $i = 1, \dots, N$, into a single vector $\vec{R} = (\vec{r}_1^T, \dots, \vec{r}_N^T)^T \in \mathbb{R}^{nN}$. If we denote the centroid of the robot deployment by $\vec{r}_c := \frac{1}{N} \sum_{i=1}^N \vec{r}_i$, then we can express $\vec{r}_i = \vec{r}_c + \vec{x}_i$ for certain vectors $\vec{x}_i \in \mathbb{R}^n$ that describe the swarm's deployment, as shown in Figure \ref{fig:2}.

\begin{figure}[h]
    \centering
    \includegraphics[scale = 0.35]{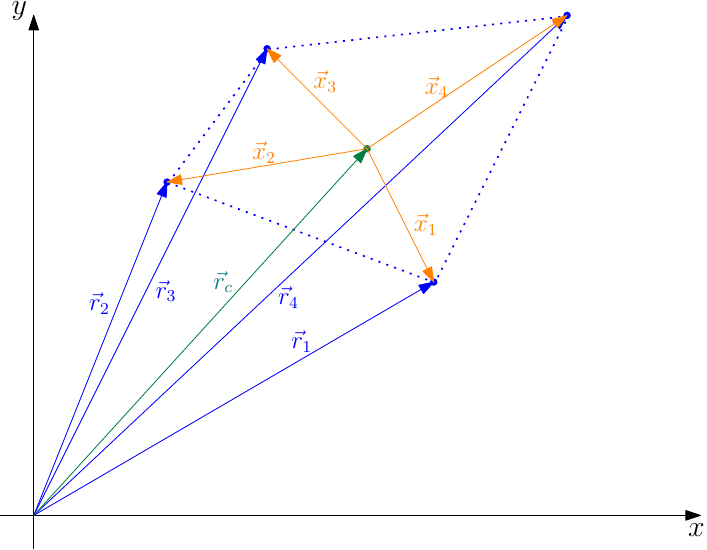}
    \caption{Deployment of a swarm of $N=4$ robots in the plane.}
    \label{fig:2}
\end{figure}

\begin{definition}
    We denote the \emph{geometry} of the swarm by the vector $\vec{X} = (\vec{x}_1^T, \dots, \vec{x}_N^T)^T$. Furthermore, we say that a geometry is \emph{non-degenerate} if the vectors $\{\vec{x}_1, \dots, \vec{x}_N\}$ span the space $\mathbb{R}^n$.
\end{definition}

Note that the non-degeneracy condition is natural, as it merely requires that the robots are positioned in such a way that allows for information extraction from the entire space. In this paper, we consider two types of dynamics for the robots: first, \emph{free dynamics} where we can directly control their velocity without restrictions, which can be expressed as
\begin{equation}
    \dot{\vec{r}}_i = \vec{u}_i(\vec{R}),
\label{eq: libre}
\end{equation}
where $\vec{u}_i(\vec{R}) \in \mathbb{R}^n$ is the velocity input or control signal; second, \emph{unicycle dynamics} with non-holonomic constraints in the plane ($n=2$), where we can only modify the direction of each robot, but not its speed. Thus, if we denote by $\alpha = (\alpha_1, \dots, \alpha_N) \in \mathbb{T}^N$ the vector that determines the velocity directions, we have
\begin{equation}\label{ec:dinámica}
    \dot{\vec{r}}_i = u_r \vec{m}_i(\alpha_i), \quad \vec{m}_i(\alpha_i) = \begin{bmatrix} \cos{(\alpha_i)} \\ \sin{(\alpha_i)} \end{bmatrix}, \quad \text{and} \quad \dot{\alpha}_i = \omega_i,
\end{equation}
where we can only act on the angular velocity $\omega_i(\vec{R}, \alpha) \in \mathbb{R}$, which determines the direction of motion $\alpha_i$ of the robot, and the speed $u_r \in \mathbb{R}^+$ is fixed. We are now ready to formalize the source-location problem.

\begin{problema}\label{prob-búsqueda}
Given a signal $\sigma$ and a swarm of $N$ robots, the search problem consists in finding a control law for the robots' actions such that, for a given $\epsilon > 0$, there exists a finite time $t_0$ such that $\norm{\vec{r}_c(t) - \vec{r}^*} < \epsilon$ for all $t \geq t_0.$
\end{problema}

\subsection{The ascending direction}
In \cite{brinon2019multirobot}, it is shown that
\begin{equation}\label{eq:grad}
    \hat{\nabla} \sigma(\vec{r}_c) = \frac{n}{ND^2} \sum_{i=1}^N \sigma (\vec{r}_i) (\vec{r}_i - \vec{r}_c),
\end{equation}
approximates the gradient at the centre of the circle/sphere in the particular case where the $N$ robots are equidistributed on a circle ($n = 2$, $N \geq 3$) or a sphere ($n=3$, $N>3$ and even) of radius $D$.

Building on this idea, we can construct an ascending direction that works for more general swarm geometries. In particular, we propose that
\begin{align}\label{eq:ascenso}
    \hat{L}_{\sigma}(\vec{r}_c,\vec{X}) &= \frac{1}{ND^2}\sum_{i=1}^{N} \sigma (\vec{r}_i) (\vec{r}_i - \vec{r}_c) \nonumber \\
    &= \frac{1}{ND^2}\sum_{i=1}^{N} \sigma (\vec{r}_c + \vec{x}_i) \vec{x}_i
\end{align}
approximates an ascending direction at the centroid, where $D = \max_{1 \leq i \leq N} ||\vec{x}_i||.$
Note that, unlike in \eqref{eq:grad} where the geometry is fixed, we must include the geometry $\vec{X}$ in the function $\hat{L}_{\sigma}$ to accommodate generic geometries. Furthermore, while \eqref{eq:grad} approximates the gradient for a specific geometry, \eqref{eq:ascenso} will approximate an ascending direction for a much wider range of geometries.

Until now, we have used the term \emph{approximation} somewhat loosely, and we will now provide a more precise notion of this. First, note that from the conditions in Definition \ref{def:señal}, it is straightforward to verify that there exist constants $K, M \in \mathbb{R}^+$ such that
\begin{equation}\label{eq:cotas2}
    || \nabla \sigma (\vec{r}) || \leq K \quad \text{and} \quad || H_\sigma(\vec{r}) || \leq 2M, \quad \forall \vec{r} \in \mathbb{R}^2,
\end{equation}
where $H_\sigma$ denotes the Hessian of the signal $\sigma$. Thus, for any $\vec{r}, \vec{r}_c \in \mathbb{R}^n$, it follows that
\begin{equation}\label{eq:cota1}
    \left| \sigma(\vec{r}) - \sigma(\vec{r}_c) - \nabla \sigma(\vec{r}_c) \cdot (\vec{r} - \vec{r}_c) \right| \leq M \norm{\vec{r} - \vec{r}_c}^2.
\end{equation}

We now specify what we mean by approximation.

\begin{prop}\label{th:dir-descenso}
    Let $\sigma$ be a signal distribution and $\vec{R} = (\vec{r}_1^T, \dots, \vec{r}_N^T)^T$ be a swarm of robots. Then we have
    \begin{equation*}
        \norm{\hat{L}_{\sigma}(\vec{r}_c, \vec{X}) - \vec{L}_{\sigma}(\vec{r}_c, \vec{X})} \leq MD
    \end{equation*}
    where
    \begin{equation*}
        \vec{L}_{\sigma}(\vec{r}_c, \vec{X}) = \frac{1}{ND^2} \sum_{i=1}^N \left( \nabla \sigma(\vec{r}_c) \cdot \vec{x}_i \right) \vec{x}_i
    \end{equation*}
    is an ascending direction at the centroid provided $\nabla \sigma (\vec{r}_c) \neq 0$ and the geometry $\vec{X}$ is non-degenerate.
\end{prop}
\begin{proof}
For the first part note that
\begin{equation*}
    \sum_{i=1}^N \vec{x}_i = 0 \implies \sum_{i=1}^n \sigma(\vec{r}_c) \vec{x}_i = 0
\end{equation*}
combined with \eqref{eq:cota1} gives an upper-bound for $\norm{\hat{L}_{\sigma}(\vec{r}_c, \vec{X}) - \vec{L}_{\sigma}(\vec{r}_c, \vec{X})}$ in the form
\begin{align*}
    \frac{1}{ND^2} \norm{\sum_{i=1}^N [\sigma(\vec{r}_c + \vec{x}_i) - \sigma(\vec{r}_c) - \nabla \sigma(\vec{r}_c) \cdot \vec{x}_i] \vec{x}_i} \\
    \stackrel{\eqref{eq:cota1}}{\leq} \frac{1}{ND^2} \sum_{i=1}^N M ||\vec{x}_i||^3 \leq MD.
\end{align*}
To see that this is an ascending direction, observe that
\begin{equation*}
    \nabla \sigma(\vec{r}_c) \cdot \vec{L}_{\sigma}(\vec{r}_c, \vec{X}) = \frac{1}{ND^2} \sum_{i=1}^N \left| \nabla \sigma(\vec{r}_c) \cdot \vec{x}_i \right|^2 > 0,
\end{equation*}
where the inequality is trivial provided $\nabla \sigma(\vec{r}_c) \neq 0$ and $\{\vec{x_1}, \dots, \vec{x}_N\}$ span the space.
\end{proof}

Proposition \ref{th:dir-descenso} guarantees that the distance between the two directions $\hat{L}_{\sigma}(\vec{r}_c, \vec{X})$ and $\vec{L}_{\sigma}(\vec{r}_c, \vec{X})$ decreases linearly with $D$, allowing us to make this distance arbitrarily small.

\begin{rem}
It is worth noting that if the robots are uniformly distributed on a circle, the previous result simplifies, up to a factor, to Theorem 1 in \cite{brinon2019multirobot}, and we obtain $\vec{L}_{\sigma}(\vec{r}_c, \vec{X}) = \frac{1}{2} \nabla \sigma(\vec{r}_c)$.
\end{rem}

It is interesting to observe that the ascending direction $\vec{L}_{\sigma}(\vec{r}_c, \vec{X})$ \emph{controls} the gradient, provided that the geometry $\vec{X}$ remains constant. This is illustrated by the following lemma.

\begin{lemma}\label{lema-Ltograd}
If the geometry $\vec{X}$ is non-degenerate, then there exists a constant $C(\vec{X}) > 0$, which depends solely on the swarm's geometry, such that
\begin{equation*}
    \frac{1}{C(\vec{X})} \|\nabla \sigma(\vec{r}_c)\|^2 \leq \vec{L}_{\sigma}(\vec{r}_c, \vec{X}) \cdot \nabla \sigma(\vec{r}_c) \leq C(\vec{X}) \|\nabla \sigma(\vec{r}_c)\|^2.
\end{equation*}
\end{lemma}
\begin{proof}
    If $\nabla \sigma(\vec{r}_c) = 0$, the result is immediate. Otherwise, we have
    \begin{equation*}
        \vec{L}_{\sigma}(\vec{r}_c, \vec{X}) \cdot \nabla \sigma(\vec{r}_c) = \frac{1}{ND^2} \sum_{i=1}^N \left| \nabla \sigma(\vec{r}_c) \cdot \vec{x}_i \right|^2 > 0,
    \end{equation*}
    because the geometry is non-degenerate. Since the geometry $\vec{X}$ is fixed, the above scalar product can be considered as a continuous function $L: \mathbb{R}^n \to \mathbb{R}$, where $\nabla \sigma(\vec{r}_c)$ is now a variable parameter. By homogeneity in the inequalities, it is enough to analyze the case $\|\nabla \sigma(\vec{r}_c)\| = 1$, i.e., vectors on the unit sphere $\mathbb{S}^{n-1}$, and to show that there exists a constant $C(\vec{X}) > 0$ such that
    \begin{equation*}
        \frac{1}{C(\vec{X})} < \vec{L}_{\sigma}(\vec{r}_c, \vec{X}) \cdot \nabla \sigma(\vec{r}_c) < C(\vec{X}).
    \end{equation*}
    This is immediate since the unit sphere is compact and the function $L$ is continuous, in particular, it attains a maximum $M^*$ and a minimum $M_* > 0$, as it never becomes zero. Taking $C > \max{\{M_*, 1/M_*\}}$ gives the result.
\end{proof}

Note that this control indicates that, up to a constant factor, following the ascending direction $\vec{L}_{\sigma}(\vec{r}_c, \vec{X})$ is as effective as following the direction given by the gradient.

\section{Convergence results}

\subsection{Free Dynamics case}

In this section, we will consider the dynamics given by equation (\ref{eq: libre}) and assume that the robot swarm can directly follow the field $\vec{L}_{\sigma}(\vec{R}, \vec{X})$. We will show that under these conditions the centroid of the swarm converges to the source of the field $\sigma$. Furthermore, as an extension to practical real-world implementations, we will demonstrate that the approximation $\hat{L}_{\sigma}$ is also effective.

\begin{theorem}\label{th:convergencia-1}
Let $\sigma$ be a signal and $\vec{R}_0 = (\vec{r}_{1,0}^T, \dots, \vec{r}_{N,0}^T)^T$ be a swarm of $N$ robots with a non-degenerate geometry. Then the free dynamic
\begin{align*}
    \dot{\vec{r}}_i(t) &= \vec{L}_{\sigma}(\vec{r}_c(t), \vec{X}(t)), \hspace{10pt} i = 1, \dots, N \\
    \vec{r}_i(0) &= \vec{r}_{i,0},
\end{align*}
gives a solution to the search problem \ref{prob-búsqueda}.
\end{theorem}

\begin{proof}
First, observe that $\vec{X}(0) = \vec{X}(t)$ for all $t > 0$ because the robots are moving in unison. Therefore, we can consider the Lyapunov function $V(\vec{r}_c) = \sigma(\vec{r}_c)$, such that throughout the trajectory we have
\begin{align*}
    \dot{V}(t) &= \dot{\vec{r}}_c(t) \cdot \nabla \sigma(\vec{r}_c(t)) \\
    &= \frac{1}{N} \sum_{i=1}^{N} \dot{\vec{r}}_i(t) \cdot \nabla \sigma(\vec{r}_c(t)) \\
    &= \frac{1}{N} \sum_{i=1}^{N} \vec{L}_{\sigma}(\vec{r}_c(t), \vec{X}(0)) \cdot \nabla \sigma(\vec{r}_c(t)) \\
    &= \vec{L}_{\sigma}(\vec{r}_c(t), \vec{X}(0)) \cdot \nabla \sigma(\vec{r}_c(t)) \geq 0.
\end{align*}
Since the geometry $\vec{X}(0)$ is non-degenerate, equality holds if and only if $\nabla \sigma(\vec{r}_c(t)) = 0$, which implies that $\vec{r}_c(t) = \vec{r}^*$. Consequently, $V(t)$ is non-decreasing and bounded throughout the trajectory. Given that the Hessian of $\sigma$ is bounded, $\ddot{V}(t)$ is also uniformly bounded. Therefore, by applying Barbalat's lemma in conjunction with the LaSalle invariance principle \cite{khalil2015nonlinear, nikravesh2018nonlinear}, we can conclude that $\lim_{t \to \infty} \vec{r}_c(t) = \vec{r}^*$.
\end{proof}


\begin{corollary}\label{corollary-convergencia1}
For any $\epsilon > 0$ there exists $D^*(\epsilon)$ such that, whenever $D \leq D^*$, Theorem \ref{th:convergencia-1} remains a solution to the search problem when $\vec{L}_\sigma$ is replaced by $\hat{L}_\sigma$.
\end{corollary}

\begin{proof}
We can write $\vec{L}_\sigma = \hat{L}_\sigma + \vec{E}$, where $\vec{E}$ represents an error term with $\norm{\vec{E}} \leq MD$, as established in Proposition \ref{th:dir-descenso}. By choosing $D$ sufficiently small, it is straightforward to verify that convergence to an $\epsilon$-neighborhood in finite time is maintained in the proof of Theorem \ref{th:convergencia-1}.
\end{proof}

\subsection{Non-Holonomic Swarm}\label{sec: non-holon}

In this section, we extend our analysis by considering the constant-speed unicycle dynamics in the plane, as described in \eqref{ec:dinámica}. For clarity, we defer the formal proofs of the results to Appendix \ref{appendix}.

To guide the movement of our system, we will draw inspiration from the method of \emph{guiding fields}, as developed in \cite{kapitanyuk2017guiding,de2017guidance}. However, these methods depend on global knowledge of the field to compute quantities (such as gradients and descent directions) across the entire space. In our case, there are two challenges: first, we do not know the field until it has been explored by the trajectory; second, the field we wish to follow may depend on the geometry of the robot swarm, not just the spatial point of the centroid.

To understand the latter, note that Theorem \ref{th:dir-descenso} provides a field to follow in order to find the maximum, but this ascending direction depends not only on the point $\vec{r}_c$ but also on the swarm geometry $\vec{X}$.

For practical purposes, it is convenient to normalize the guiding field. We consider the open set $U \subseteq \mathbb{R}^{2N}$ characterized by $\vec{L}_\sigma(\vec{R}) \neq 0$ and define $\vec{m}_d: U \to \mathbb{R}^2$ by
\begin{equation}\label{eq:campoguia}
    \vec{m}_d(\vec{R}) = \frac{\vec{L}_\sigma(\vec{R})}{\norm{\vec{L}_\sigma(\vec{R})}},
\end{equation}
so that $\norm{\vec{m}_d(\vec{R})} = 1$. This field will be referred to as the \emph{guiding field} and is the field we wish to follow to find the maximum.

Intuitively, to follow the field, the natural approach is to reduce the angle between each robot’s velocity and the direction of the field. To formalize this idea, we need to define new concepts and prove certain preliminary results. We begin by studying the behaviour of the derivative of the guiding field.

\begin{lemma}\label{lema:derivativeCGI}
Let $\vec{m}_d(t) = \vec{m}_d(\vec{R}(t))$ be the guiding field throughout the swarm's trajectory. Then
\begin{equation*}
    \dot{\vec{m}}_d(t) = \omega_{d}(\vec{R}(t), \alpha(t)) E \vec{m}_d(t),
\end{equation*}
where $\omega_{d}: \mathbb{R}^2 \to \mathbb{R}$ is a continuous, uniquely determined function and $E = \begin{bmatrix} 0 & -1 \\ 1 & 0 \end{bmatrix}$.
\end{lemma}

We introduce the directed angle, $\delta_i (\vec{r}, \alpha_i) \in (-\pi, \pi]$, between $\vec{m}_d$ and $\vec{m}_i$. The function $\delta_i$ is defined at any point $(\vec{R}, \alpha)$ where $\vec{m}_d$ is defined and is of class $C^1$ at points where $\vec{m}_i (\alpha_i) \ne -\vec{m}_d(\vec{R})$, as depicted in Figure \ref{fig:3}.

\begin{figure}[h]
    \centering
    \includegraphics[width = 4cm, height = 4cm]{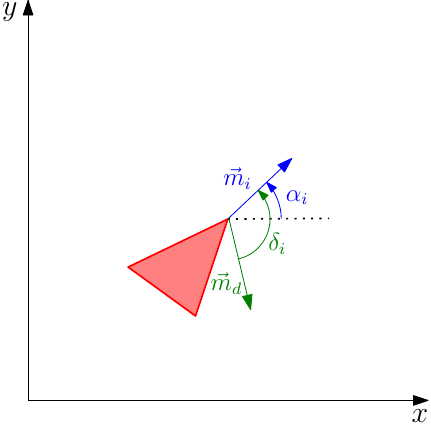}
    \caption{Orientation of a robot and directed angles.}
    \label{fig:3}
\end{figure}

It is straightforward to observe that throughout the trajectory, the change in $\delta_i$ will be governed by both the change in $\alpha_i$ and the variation in the angle subtended by the field $\vec{m}_d$. Hence, at any point where $\dot{\delta}_i(t)$ is defined, i.e., where $\delta_i(t) < \pi$, the following fundamental identity holds
\begin{equation}\label{eq:fundamental}
    \dot{\delta}_i = \omega_i - \omega_d.
\end{equation}

Using this relationship, we can address the issue of discontinuities in $\delta_i$ within our system by introducing an augmented system that includes the parameter $\delta_{i,*}$. This leads to the system

\begin{align}
\begin{cases}
    \dot{\Vec{r}}_i &= u_r \begin{bmatrix} \cos{(\alpha_i)} & \sin{(\alpha_i)} \end{bmatrix}^T  \\
    \dot{\alpha}_i &= \omega_i \hspace{10 pt}, \quad i = 1, \dots, N  \\
    \dot{\delta}_{i, *} &= \omega_i - \omega_d(\vec{r}), \hspace{10 pt} \delta_{i,*}(0) = \delta_{i}(0),
\end{cases} \label{eq:augmentado}
\end{align}
where the fundamental identity \eqref{eq:fundamental}, together with the initial condition $\delta_{i,*}(0) = \delta_{i}(0)$, ensures that the solutions to the augmented system will also give a solution to \eqref{ec:dinámica}. Thus, within this new system, we can apply appropriate existence and uniqueness theorems, provided that $\vec{X}(0)$ is non-degenerate and $\nabla \sigma(\vec{r}_c(0)) \neq 0$, since the right-hand side of the system \eqref{eq:augmentado} is locally $L$-Lipschitz in its variables \cite{khalil2015nonlinear, arnold1992ordinary}.

Note that the inclusion of the parameter $\delta_{i,*}$ is purely formal; since we are working in a specific trajectory, this new definition allows us to globally define the oriented angle throughout it, thereby eliminating the pathological discontinuity that occurs when $\vec{m}_i(\alpha_i) = -\vec{m}_d(\vec{R})$.

We aim to ensure that the oriented angle $\delta_{i,*}$ tends to zero so that the velocities and the guiding field align. The most natural approach to achieve this is to choose a control parameter $\omega_i$ that is proportional to $\delta_{i,*}$ and possibly includes an additional term. Specifically, we can design
\begin{equation*}
    \omega_i = -k_\gamma \delta_{i,*} + A_i,
\end{equation*}
where $k_\gamma$ is a positive constant and $A_i$ is a parameter that could potentially be chosen. This gives
\begin{equation*}
    \dot{\delta}_{i,*} = \omega_i - \omega_d = - k_\gamma \delta_{i,*} + (A_i - \omega_d).
\end{equation*}
Herein lies the challenge compared to traditional methods. If we had global knowledge of the field, we could set $A_i = \omega_d$, allowing $\delta_{i,*}$ to decay exponentially to zero. However, since the field is unknown, we cannot determine $\omega_d(t)$. There are two possible approaches to address this issue:

\begin{itemize}
\item We could approximate $A_i \approx \omega_d$, for example, by using data from previously traversed trajectories or by utilizing multiple nearby robot swarms, enabling us to estimate the field’s distortion from the combined data.

\item Alternatively, as we will do in this work, we can disregard the term $A_i$ (by setting it to zero) and try to bound the value of $\omega_d$, such that for sufficiently high values of $k_\gamma$, exponential decay in $\delta_{i,*}$ is still achieved.
\end{itemize}

Note that seeking to ensure that $\omega_d$ is bounded is a natural condition, as this merely means that the field given by $\vec{m}_d$ does not change too abruptly. Specifically, we want the angular velocity $\omega_i$ to be fast enough to track the changes in the field. Henceforth, we will consider the system \eqref{eq:augmentado} with
\begin{equation*}
    \omega_i = -k_\gamma \delta_{i,*}, \hspace{10pt} i = 1, \dots, N,
\end{equation*}
where $k_\gamma$ is a positive constant that will be determined later. With this clarification, we can demonstrate the algorithm's effectiveness for robots with constant-speed unicycle dynamics, leading to the following convergence result.

\begin{theorem}\label{th: main-unic}
    Let $\sigma$ be a signal distribution, $\vec{R}_0 = (\vec{r}_{1,0}, \dots, \vec{r}_{N,0})$ be a swarm with non-degenerate geometry, initial velocity directions $\alpha_0 = (\alpha_{1,0}, \dots, \alpha_{N,0})$, and $\epsilon > 0$ fixed. Then there exists a constant $k_\gamma$ such that for all $i = 1, \dots, N$, the dynamics
    \begin{align*}
\begin{cases}
    \dot{\Vec{r}}_i &= u_r \begin{bmatrix} \cos{(\alpha_i)} \\ \sin{(\alpha_i)} \end{bmatrix}, \hspace{10 pt} \vec{r}_i(0) = \vec{r}_{i,0} \nonumber \\
    \dot{\alpha}_i &= -k_\gamma \delta_{i,*}, \hspace{10 pt} \alpha_i(0) = \alpha_{i,0} \\
    \dot{\delta}_{i, *} &= -k_\gamma \delta_{i,*} - \omega_d(\vec{r}), \hspace{10 pt} \delta_{i,*}(0) = \delta_{i}(0),
\end{cases}
    \end{align*}
    is a solution to the search problem \ref{prob-búsqueda}.
\label{th: mainresult}
\end{theorem}

Similar to the case with free dynamics, the field given by (\ref{eq:campoguia}) cannot be directly measured by the robots. However, the approximation
\begin{equation}
    \hat{m}_d = \frac{\hat{L}_\sigma(\vec{r}_c, \vec{X})}{\norm{\hat{L}_\sigma(\vec{r}_c, \vec{X})}},
\label{eq:campoaprox}
\end{equation}
is measurable and suffices in most cases. As a corollary to Theorem \ref{th: mainresult}, one could demonstrate the algorithm's effectiveness in finding the maximum of $\sigma$ using the approximated field $\hat{m}_d$ as it was done in the free-dynamics case.

\section{Numerical simulations}

In this section, we discuss the performance of our algorithm in light of numerical verification of the analytical predictions. The implementation of the simulations is divided into two steps: first, we generate random robot swarms, and then, we execute the source-locating algorithm previously described. 

In the first step, we randomly select points in space\footnote{There are many methods to do so. Here, when referring to randomness, we mean normally distributed coordinates.} which will act as formation centres. Five robots are then placed around these centroids, following a random distribution both in terms of their distance from the centre and their angular distribution around it. Consequently, each point results in a swarm of robots with highly varied geometries.

In the second step, each swarm solves the differential system given by Corollary \ref{corollary-convergencia1}, i.e., using $\hat{L}_{\sigma}$, which is the computable value for the free-dynamics robot system, and Theorem \ref{th: mainresult} for swarms with constant-speed unicycle dynamics, where $\vec{m}_d$ is appropriately replaced by $\hat{m}_d$, which is the computable value for the unicycle-dynamics robot system.

For the scalar field $\sigma$, we chose two-dimensional Gaussian functions with a peak at the origin and a standard deviation of ten. Convergence was also verified for signals $\sigma$ represented by quadratic forms with maxima at the origin.

To numerically solve the differential systems, we used the \emph{solve\_ivp} method from \emph{SciPy}, which is based on the Runge-Kutta $45$ method \cite{DORMAND198019}.

\subsection{Numerical simulation of free-dynamics systems}
\label{sec-simulaciones1}

Figure \ref{fig:tray1} shows the trajectories of swarms under a Gaussian signal distribution. The red point marks the initial position of the centroid for each swarm with free dynamics. The figure further demonstrates how the distance between the swarm's centroid and the field's maximum decreases over time, consistent with the behaviour expected from a gradient descent-like approach.

\begin{figure}
    \centering
    \includegraphics[width=0.9\columnwidth]{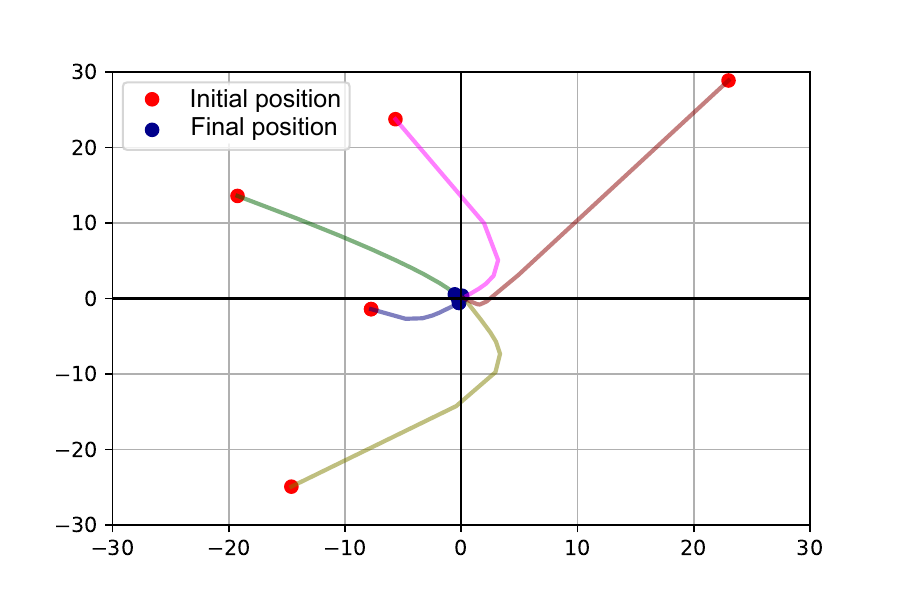}
    \includegraphics[width=0.9\columnwidth]{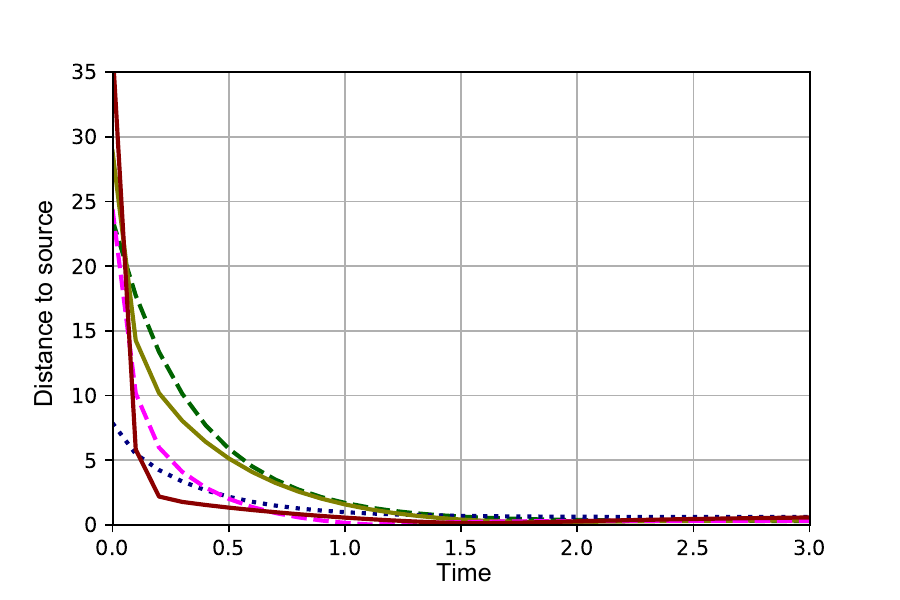}
    \caption{Convergence under a Gaussian signal for swarms with free dynamics randomly distributed in the plane. The top figure details the trajectories of centroids of various swarms with arbitrary non-degenerate geometries. The bottom figure shows the evolution of the distances between the maximum of the field and the centroids of the swarms. Time and distance units are arbitrary.}%
    \label{fig:tray1}%
\end{figure}

It is also noteworthy to examine the behaviour of robot swarms with nearly degenerate geometries, where the swarm's configuration is close to a straight line. Figure \ref{fig:tray2} illustrates examples of trajectories for such nearly degenerate geometries in the horizontal and vertical axes. The figure shows not only the centroid of the swarm (indicated by a red point) but also each individual robot (represented by yellow points) together with their trajectories. In these scenarios, analysis of the equations reveals that convergence occurs rapidly in the direction of degeneration, followed by a more gradual convergence in the perpendicular direction. Overall, the algorithm exhibits high stability across a variety of geometries, including those that are nearly degenerate.

\begin{figure}
    \centering
    \includegraphics[width=0.9\columnwidth]{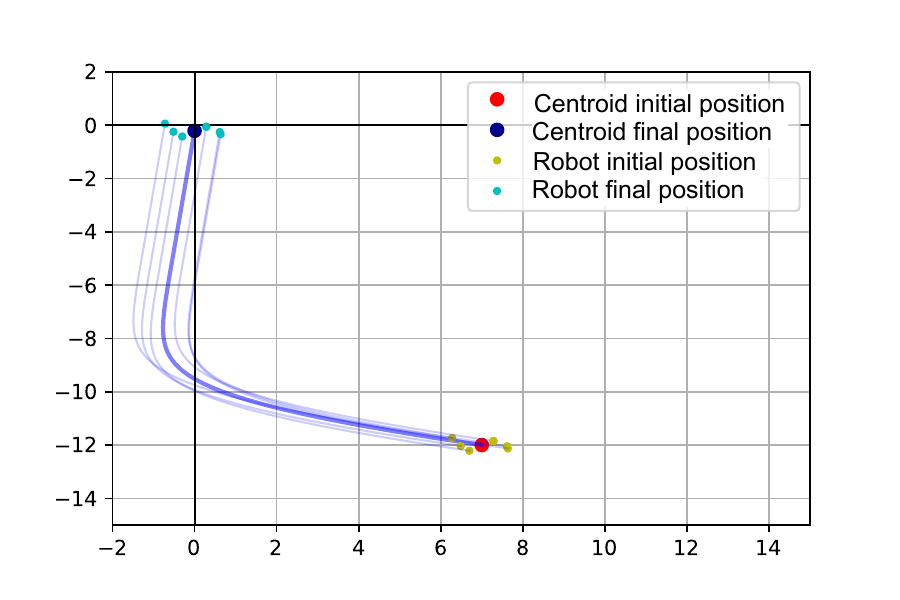} %
    \includegraphics[width=0.9\columnwidth]{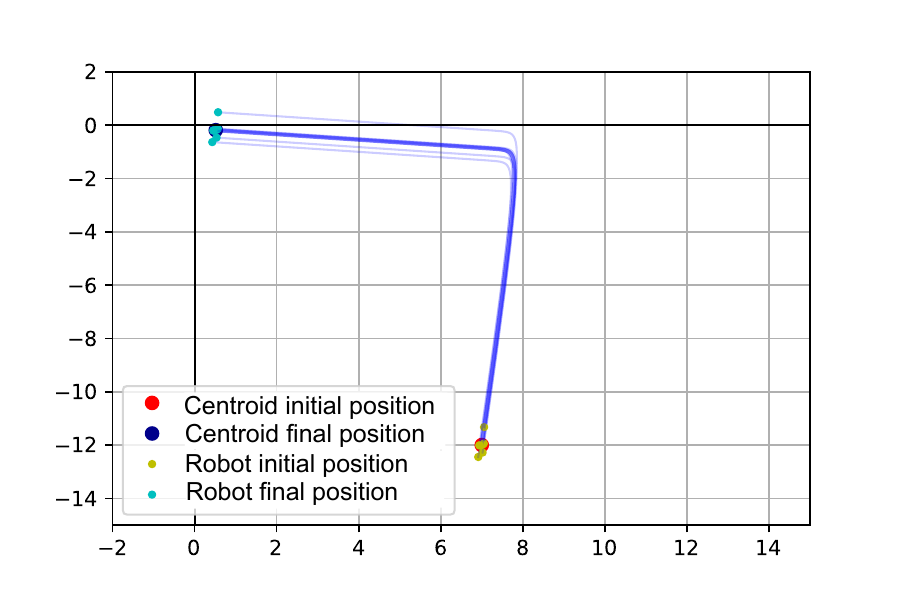} %
    \caption{Convergence of robot swarms with nearly degenerate geometries. The top figure shows a trajectory with nearly degenerate geometry in the horizontal axis. The bottom figure illustrates a trajectory with nearly degenerate geometry in the vertical axis.} %
    \label{fig:tray2}%
\end{figure}

\subsection{Numerical simulation of constant-speed unicycle dynamics}

Analogous to the scenario with robots displaying free dynamics, we aim to accurately represent the information measured by the robots swarm by employing the guiding field approximation given in (\ref{eq:campoaprox}), which serves as an approximate version of (\ref{eq:campoguia}). Similarly to Section \ref{sec-simulaciones1}, we will generate multiple swarms of robots, each with initial velocity directions drawn from a normal distribution. Consequently, at the beginning of the simulation, each robot in the swarm will be oriented in a distinct and random direction. As previously, we will use Gaussian and quadratic signals $\sigma$ with maxima located at the origin.

For the parameter $k_\gamma$, we selected a value such that $u_r / k_\gamma = 1$. Although this might initially seem insufficient to meet the theoretical conditions, practical results have demonstrated its effectiveness.

As previously discussed, the technical conditions are less restrictive than they might appear at first glance. It is also important to note that if $k_\gamma$ were to approach infinity, the results would closely resemble those observed in Section \ref{sec-simulaciones1}, indicating that the unicycles would display behaviour similar to free dynamics. Therefore, this moderately low choice for $k_\gamma$ allows us to examine scenarios with more realistic robot behaviour.

Figure \ref{fig:tray3} illustrates that, unlike the scenario with free dynamics, the centroids do not constantly converge to the maximum but instead tend to orbit around it. This behaviour is better observed in the lower graph of Figure \ref{fig:tray3}, where, once close enough to the maximum, the centroid exhibits periodic oscillations around it.

\begin{figure}
    \centering
    \includegraphics[width=0.9\columnwidth]{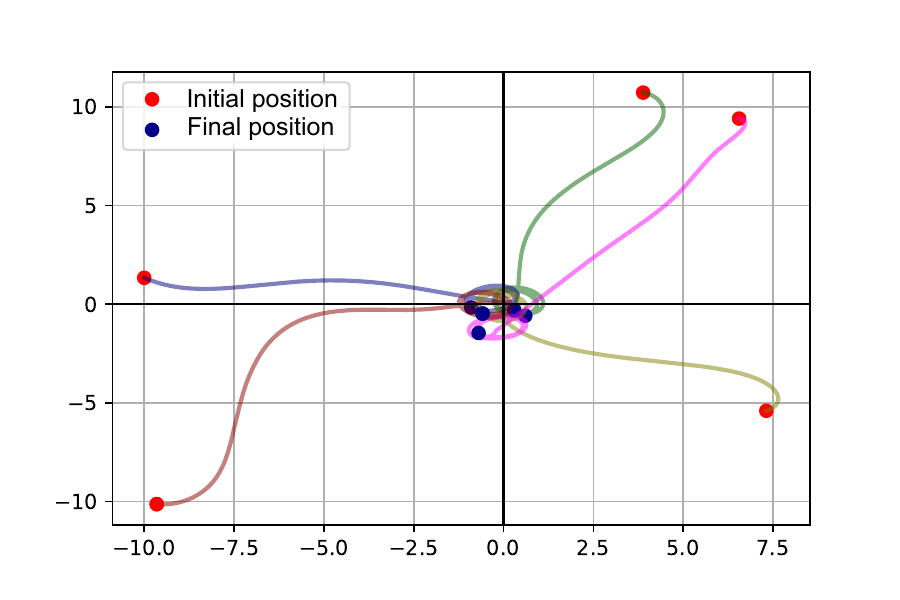} \\%
    \includegraphics[width=0.9\columnwidth]{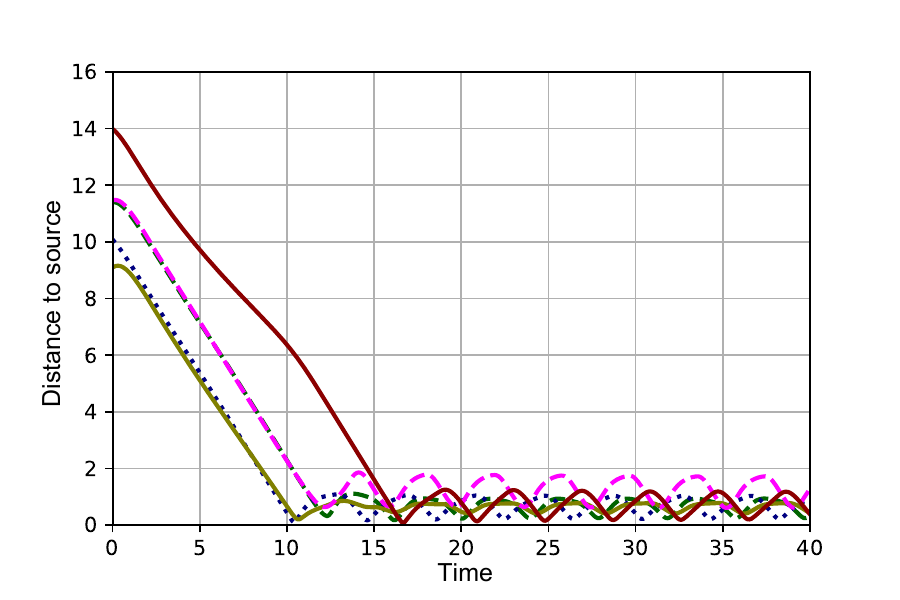}%
    \caption{Convergence of robot swarms with constant-speed unicycle dynamics for a Gaussian signal, using the guiding field algorithm (\ref{eq:campoaprox}). }%
    \label{fig:tray3}%
\end{figure}

In Figure \ref{fig:tray4}, we observe the behaviour of the same swarm geometry under different values of \( k_\gamma \). In the upper figure, a relatively low value of \( k_\gamma \) results in a final geometry that is significantly different from the initial configuration. In contrast, the lower one shows that a high value of \( k_\gamma \) maintains the initial swarm geometry almost unchanged. Moreover, note how the majority of the deformation in the upper graph occurs initially, after which the swarm tends to move almost in unison, as guaranteed by the theoretical lemmas in the appendix.

\begin{figure}
    \centering
    \includegraphics[width=0.9\columnwidth]{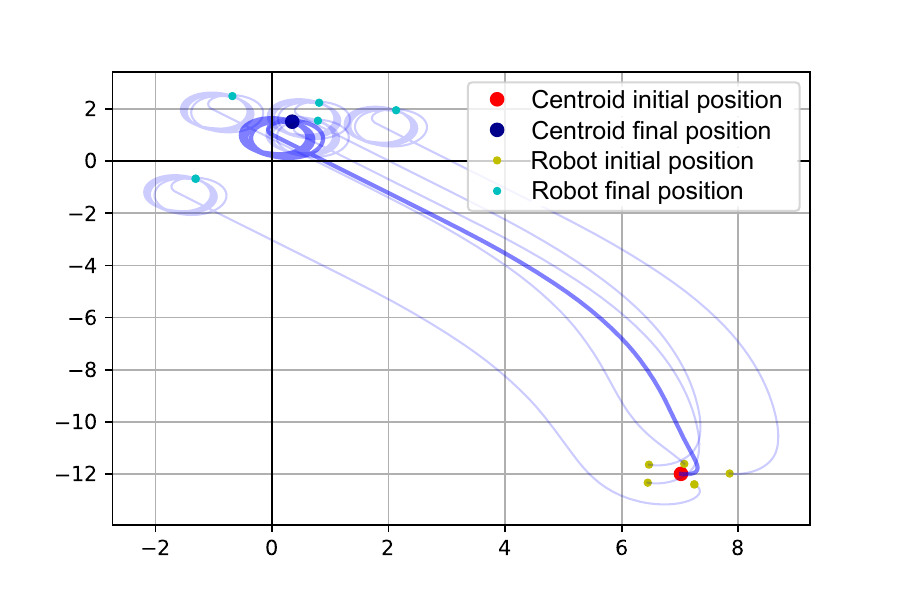} \\
    \includegraphics[width=0.9\columnwidth]{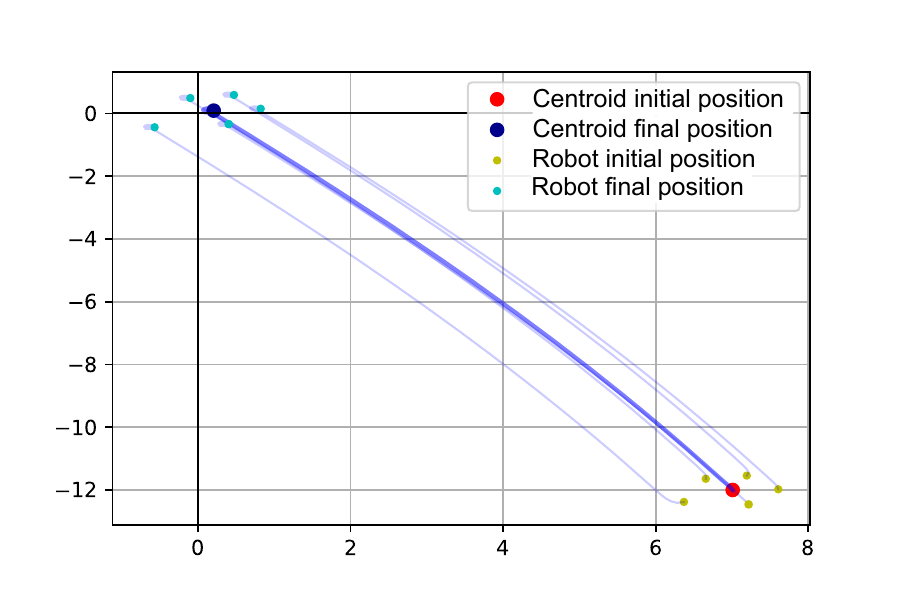}
    \caption{Two trajectories of the same robot swarm for different values of \( k_\gamma \). The upper figure uses a low value of \( k_\gamma \), resulting in a significantly different final geometry. In contrast, the lower one illustrates how a higher value of \( k_\gamma \) maintains the initial swarm geometry almost unchanged. 
    Generally, lower values of \( k_\gamma \) lead to greater differences between the initial and final geometries around the maximum of the field \( \sigma \).}
    \label{fig:tray4}
\end{figure}

\section{Conclusions}

In this paper, we introduce a novel approach to solving the source localisation problem using gradient descent methods. Unlike traditional techniques, our method does not depend on the system's geometry, which enhances its versatility and broadens its range of potential applications. Additionally, by controlling the gradient by our ascending direction, we ensured that the convergence rate of our method matches that of established gradient-based techniques. We provided analytical proof of the method's convergence, which was further validated through numerical simulations.

We applied this method to both free-dynamics and unicycle-dynamics robot swarms, adapting the field-following technique to meet the specific challenges of each scenario. Future research will focus on exploring less restrictive signal models, accommodating measurement noise, handling multiple maxima, and incorporating regions with zero gradients.

\bibliographystyle{IEEEtran}
\bibliography{biblio}

\newpage

\section{Appendix}\label{appendix}

In this appendix, we provide proofs for the results claimed in Section \ref{sec: non-holon}. We start with a proof of Lemma \ref{lema:derivativeCGI}.

\begin{proof}[Proof of Lemma \ref{lema:derivativeCGI}]
Differentiating the equality \( ||\vec{m}_d(t)||^2 = 1 \) gives \(\dot{\vec{m}}_d(t) \cdot \vec{m}_d(t) = 0\). In particular \({\vec{m}}_d(t) \perp \dot{\vec{m}}_d(t)\) at any point of the trajectory and thus \(\dot{\vec{m}}_d(t)\) is proportional to the unit vector \(E \vec{m}_d(t)\). Therefore, we can write \(\dot{\vec{m}}_d(t) = \omega_d(t) E \vec{m}_d(t)\), where the scalar factor \(\omega_d(t)\) is given by
\begin{equation*}
    \omega_d(t) = \dot{\vec{m}}_d(t)^t E \vec{m}_d(t).
\end{equation*}
It remains to show that \(\omega_d\) depends only on the trajectory \((\vec{R}(t), \alpha(t))\) and that this dependence is continuous. This follows directly from the fact that both \(\vec{m}_d(t)\) and \(\dot{\vec{m}}_d(t)\) exhibit continuous dependence on these parameters.
\end{proof}

We now move to the proof of Theorem \ref{th: main-unic}. For this, we first need to address certain technical challenges, which we will do shortly. The first of these results concerns the stability of the swarm's geometry, ensuring that, under specific conditions, we can control the extent of its deformation.

\begin{lemma}\label{lema-estabilidad}
    Suppose that $\vec{R}(t)$ is a trajectory of the differential equations system \eqref{eq:augmentado}. Then
    \begin{equation*}
        \norm{\vec{r}_{i,j}(t) - \vec{r}_{i,j}(0)} \leq 2\pi \frac{u_r}{k_\gamma}, \hspace{10 pt} \forall \hspace{3 pt} 1 \leq i < j \leq N,
    \end{equation*}
    where $\vec{r}_{i,j} = \vec{r}_i - \vec{r}_j$.
\end{lemma}

\begin{proof}
    Define $\delta_{i,j} = \alpha_j - \alpha_i$, which represents the oriented angle from the velocity of robot $i$ to that of robot $j$. Note that\footnote{Due to how we measure angles, we may measure them along the longer path.} $\delta_{i,j} \in (-2\pi, 2\pi)$, such that $\delta_{i,j} = 0$ indicates that both robots are aligned. It is straightforward to verify that $\delta_{i,j} = \delta_{j,*} - \delta_{i,*}$.
    Hence, we have
    \begin{equation*}
        \dot{\delta}_{i,j} = \omega_j - \omega_i = -k_\gamma (\delta_{j,*} - \delta_{i,*}) = - k_\gamma \delta_{i,j},
    \end{equation*}
    or equivalently, throughout the trajectory of the system
    \begin{equation*}
        \delta_{i,j} = \delta_{i,j}(0) e^{-k_\gamma t},
    \end{equation*}
    so $\delta_{i,j} \to 0$ as $t \to \infty$. Therefore, we can quantify the deformation of the system as
    \begin{align*}
        \dot{\Vec{r}}_{i,j} &= \dot{\Vec{r}}_{i} - \dot{\Vec{r}}_{j} = u_r \begin{bmatrix} \cos(\alpha_i) - \cos{(\alpha_j)}  \\ \sin(\alpha_i) - \sin(\alpha_j) \end{bmatrix} \\
        &= 2u_r \sin\left( \frac{\alpha_j - \alpha_i}{2}\right) \begin{bmatrix} \sin\left( \frac{\alpha_i + \alpha_j}{2}\right) \\ - \cos\left( \frac{\alpha_i + \alpha_j}{2}\right)\end{bmatrix} \\
        &= 2u_r \sin\left( \frac{\delta_{i,j}}{2}\right) \begin{bmatrix} \sin\left( \frac{\alpha_i + \alpha_j}{2}\right) \\ - \cos\left( \frac{\alpha_i + \alpha_j}{2}\right)\end{bmatrix} \\
        &= 2u_r \sin\left(\frac{\delta_{i,j}(0)e^{-k_\gamma t}}{2} \right) \begin{bmatrix} \sin\left( \frac{\alpha_i + \alpha_j}{2}\right) \\ - \cos\left( \frac{\alpha_i + \alpha_j}{2}\right)\end{bmatrix}.
    \end{align*}
    Thus, it follows that
    \begin{equation*}
        ||\dot{\Vec{r}}_{i,j}|| \leq 2u_r \left | \sin\left(\frac{\delta_{i,j}(0)e^{-k_\gamma t}}{2}  \right) \right| \leq u_r |\delta_{i,j}(0)| e^{-k_\gamma t},
    \end{equation*}
    indicating that the position stabilises quickly. Indeed, we can write
    \begin{equation*}
        \Vec{r}_{i,j}(t) = \Vec{r}_{i,j}(0) + \int_0^t \dot{\Vec{r}}_{i,j},
    \end{equation*}
    and hence
    \begin{align*}
        \norm{\Vec{r}_{i,j}(t) - \Vec{r}_{i,j}(0)} &= \norm{ \int_0^t \dot{\Vec{r}}_{i,j}} \leq \int_0^t ||\dot{\Vec{r}}_{i,j}|| \\
         &\leq u_r |\delta_{i,j}(0)| \int_0^t e^{-k_\gamma t} dt \\
         &\leq u_r |\delta_{i,j}(0)| \int_0^\infty e^{-k_\gamma t} dt \\
        &= \frac{u_r |\delta_{i,j}(0)|}{k_\gamma} \leq 2\pi \frac{u_r}{k_\gamma},
    \end{align*}
    where we used that $\delta_{i,j}(0) \in (-2\pi, 2\pi)$.
\end{proof}

\begin{corollary}\label{corolario-estabilidad}
    Suppose that $\vec{R}(t)$ is a trajectory of the differential equations system \eqref{eq:augmentado}. Then, it follows that
    \begin{equation*}
        \norm{\vec{x}_{i}(t) - \vec{x}_{i}(0)} \leq 2\pi \frac{u_r}{k_\gamma}.
    \end{equation*}
\end{corollary}

The above result ensures that, provided the angular velocity is fast enough relative to the linear one, the system's geometry will stabilise quickly when following the field.

The second result in this section addresses the need to bound $|\omega_d|$ during our trajectory. 
This bounding is not entirely possible because of the following reason: if $\vec{r}^*$ represents a maximum of the signal, then $\grad \sigma(\vec{r}^*) = 0$, meaning that $\vec{m}_d$ is not well-defined at that point\footnote{Technically, this applies to the entire subset of $\mathbb{R}^{2N}$ where $\vec{r}_c = \vec{r}^*$.}. More critically, in a sufficiently small neighbourhood around $\vec{r}^*$, the angular velocity $\omega_d$ changes rapidly, as $\omega_d$ undergoes an \emph{instantaneous} $180$-degree turn when passing through $\vec{r}^*$. To address these technical challenges, we restrict our attention to trajectories that avoid such neighbourhoods, ensuring that our robot system ultimately approaches these regions\footnote{Which is sufficient, as we are locating a region containing the maximum of our signal.}, though not controlling the behaviour within them.

Before we can provide a bound for $\omega_d$, we will need a preliminary result.

\begin{lemma}\label{lema-lower_bound}
    Let \(\vec{R}(t)\) be a trajectory of the system \eqref{eq:augmentado} with a non-degenerate initial geometry \(\vec{X}(0)\). Then, there exist constants \(\kappa, m > 0\) such that
    \begin{equation*}
        \norm{\vec{L}_\sigma(\vec{R}(t))} \geq m \norm{\grad \sigma(\vec{r}_c(t))}
    \end{equation*}
    when \(k_\gamma \geq \kappa\).
\end{lemma}
\begin{proof}
    Since \(\vec{X}(0)\) is non-degenerate, we can assume without loss of generality that \(\vec{x}_1(0)\) and \(\vec{x}_2(0)\) form a basis. Therefore, a small perturbation of these vectors will still form a basis and thus by Corollary \ref{corolario-estabilidad}, there exists \(\kappa_1 > 0\) such that the vectors \(\vec{x}_1(t)\) and \(\vec{x}_2(t)\) will always form a basis whenever \(k_\gamma \geq \kappa_1\).

    On one hand, we have
    \begin{equation*}
        \vec{L}_\sigma(\vec{R}) \cdot \grad \sigma(\vec{r}_c) \leq \norm{\vec{L}_\sigma(\vec{R})} \norm{\grad \sigma(\vec{r}_c)},
    \end{equation*}
    while by Lemma \ref{lema-Ltograd} and the fact that \(\vec{X}\) remains non-degenerate throughout the trajectory, there exist constants \(C(\vec{X})\) such that
    \begin{equation*}
        \vec{L}_\sigma(\vec{R}) \cdot \grad \sigma(\vec{r}_c) \geq \frac{1}{C(\vec{X})} \norm{\grad \sigma (\vec{r}_c)}^2.
    \end{equation*}
    Furthermore, \(C(\vec{X})\) can be chosen to have a continuous dependence on the geometry \(\vec{X}\). From Corollary \ref{corolario-estabilidad}, by selecting \(\kappa_2\) sufficiently large, we ensure that if \(k_\gamma \geq \kappa_2\), the geometry \(\vec{X}(t)\) is contained within a compact set \(F \subset \mathbb{R}^{2N}\). Consequently, \(C(\vec{X})\) reaches its maximum value on this compact set \(F\), for a certain geometry \(\vec{X}^*\). Thus, taking \(C^* = C(\vec{X}^*)\), it follows that throughout the trajectory
    \begin{equation*}
        \vec{L}_\sigma(\vec{R}) \cdot \grad \sigma(\vec{r}_c) \geq \frac{1}{C(\vec{X})} \norm{\grad \sigma (\vec{r}_c)}^2 \geq \frac{1}{C^*} \norm{\grad \sigma (\vec{r}_c)}^2.
    \end{equation*}
    Combining this with the previous inequality, we get
    \begin{equation*}
        \frac{1}{C^*} \norm{\grad \sigma (\vec{r}_c)}^2 \leq \vec{L}_\sigma(\vec{R}) \cdot \grad \sigma(\vec{r}_c) \leq \norm{\vec{L}_\sigma(\vec{R})} \norm{\grad \sigma(\vec{r}_c)},
    \end{equation*}
    which implies
    \begin{equation*}
        \frac{1}{C^*} \norm{\grad \sigma (\vec{r}_c)} \leq \norm{\vec{L}_\sigma(\vec{R})}.
    \end{equation*}
    Setting \(m = \frac{1}{C^*}\) and \(\kappa = \max \{\kappa_1, \kappa_2\}\) gives the result.
\end{proof}

Equipped with this result, we are now ready to give a bound on the angular velocity.

\begin{lemma}\label{lema-cotaomega}
     For every \(\epsilon > 0\) there exist constants \(\kappa\), \(\Omega_d(\epsilon) > 0\) such that for any $k_\gamma \geq \kappa$ and every trajectory \(\vec{R}(t)\) of the system \eqref{eq:augmentado} with a non-degenerate initial geometry \(\vec{X}(0)\), then  \(\norm{\grad \sigma (\vec{r}_c(t))} \geq \epsilon\) implies \(|\omega_d| \leq \Omega_d\).
\end{lemma}

\begin{proof}
    Recall that
    \begin{equation*}
        \omega_d(t) = \dot{\Vec{m}}_d(t)^T E \Vec{m}_d(t),
    \end{equation*}
    which implies that
    \begin{equation*}
        |\omega_d(t)| \leq \norm{\dot{\Vec{m}}_d(t)} \norm{\Vec{m}_d(t)} = \norm{\dot{\Vec{m}}_d(t)},
    \end{equation*}
    so a bound on \(\norm{\dot{\Vec{m}}_d(t)}\) gives a bound on \(\omega_d\).

    From the definition of \(\norm{\dot{\Vec{m}}_d(t)}\) we have
    \begin{equation*}
       \norm{\dot{\Vec{m}}_d(t)} = \frac{\sum_{i=1}^N (\vec{\nabla} \sigma (\vec{r}_c) \cdot (\vec{r}_i - \vec{r}_c)) (\vec{r}_i - \vec{r}_c)}{\norm{\sum_{i=1}^N (\vec{\nabla} \sigma (\vec{r}_c) \cdot (\vec{r}_i - \vec{r}_c)) (\vec{r}_i - \vec{r}_c)}},
    \end{equation*}
    and differentiating one of the terms in the sum gives
    \begin{align*}
        \frac{d}{dt}&\left( \vec{\nabla} \sigma (\vec{r}_c) \cdot (\vec{r}_i - \vec{r}_c) (\vec{r}_i - \vec{r}_c) \right) \\
        =& \left( H_\sigma (\vec{r}_c) \cdot (\vec{r}_i - \vec{r}_c) + \vec{\nabla} \sigma (\vec{r}_c) \cdot (\dot{\vec{r}}_i - \dot{\vec{r}}_c) \right) (\vec{r}_i - \vec{r}_c) \\
        &+ \vec{\nabla} \sigma (\vec{r}_c) \cdot (\vec{r}_i - \vec{r}_c) (\dot{\vec{r}}_i - \dot{\vec{r}}_c).
    \end{align*}

    By Corollary \ref{corolario-estabilidad}, on the trajectory we have \(\norm{\vec{r}_i - \vec{r}_c} = \norm{\vec{x}_i} \leq \frac{2\pi u_r}{k_\gamma} + \norm{\vec{x}_i(0)} \leq M_1\), provided \(k_\gamma \geq 1\). Also,
    \begin{equation*}
        \norm{\dot{\vec{r}}_i - \dot{\vec{r}}_c} = \norm{\dot{\vec{r}}_i - \frac{1}{N} \sum_{j=1}^N \dot{\vec{r}}_j} \leq u_r + u_r = 2u_r,
    \end{equation*}
    which is easily bounded. Finally, recall that the Hessian and gradient are globally bounded by
    \begin{equation*}
        \norm{\vec{\nabla} \sigma (\vec{r})} \leq K \quad \text{and} \quad \norm{H_\sigma(\vec{r})} \leq 2M, \quad \forall \vec{r} \in \mathbb{R}^2.
    \end{equation*}
    Thus, it follows that
    \begin{equation*}
        \norm{\frac{d \vec{L}_\sigma(\vec{R})}{dt}} \leq N \left[ 2M M_1^2 + 4u_r K M_1 \right] \leq C_1^*,
    \end{equation*}
    for some global constant \(C_1^*\). On the other hand,
    \begin{equation*}
        \frac{d}{dt} \norm{\vec{L}_\sigma(\vec{R})} = \frac{1}{2 \norm{\vec{L}_\sigma(\vec{R})}} \frac{d}{dt} \norm{\vec{L}_\sigma(\vec{R})}^2 = \frac{\dot{\vec{L}}_\sigma(\vec{r}) \cdot \vec{L}_\sigma(\vec{R})}{\norm{\vec{L}_\sigma(\vec{R})}},
    \end{equation*}
    so
    \begin{equation*}
        \frac{d}{dt} \norm{\vec{L}_\sigma(\vec{R})} \leq \norm{\frac{d \vec{L}_\sigma(\vec{R})}{dt}} \leq C_1^*.
    \end{equation*}
    Combining the above equations, we obtain
    \begin{align*}
        \norm{\frac{d}{dt} \vec{m}_d} &\leq \frac{2 C_1^* \norm{\vec{L}_\sigma(\vec{R})}}{\norm{\vec{L}_\sigma(\vec{R})}^2} = \frac{2 C_1^*}{\norm{\vec{L}_\sigma(\vec{R})}} \leq \frac{2 C_1^*}{m \epsilon} = \Omega_d,
    \end{align*}
    where the last inequality follows by choosing \(k_\gamma \geq \kappa_1\) as given by Lemma \ref{lema-lower_bound}, along with the assumption that \(\norm{\grad \sigma (\vec{r}_c(t))} \geq \epsilon\). Setting \(\kappa = \max \{\kappa_1, 1\}\) yields the result.
\end{proof}

This lemma will be highly useful in conjunction with the following result, which will ultimately provide a means to control the angular difference \(\delta_{i,*}\) throughout the trajectory.

\begin{lemma}\label{lema-gamma}
    For any angle \(0 < \gamma < \pi/2\) and any \(\epsilon > 0\), there exists a constant \(\kappa > 0\) and a time \(t^*\) such that if \(k_\gamma \geq \kappa\) and \(|| \grad \sigma(\vec{r}_c(0))|| > \epsilon\), then \(|\delta_{i,*}(t^*)| < \gamma\) for all \(i = 1, \dots, N\). Furthermore, for \(t_0 \geq t^*\), this bound remains valid as long as \(|| \grad \sigma(\vec{r}_c(t))|| > \epsilon\) for all \(t \in [t^*, t_0)\).
\end{lemma}

\begin{proof}
    Let \(\gamma\) and \(\epsilon > 0\) be fixed and let \(A = \left\{\vec{y}\in \mathbb{R}^2 \mid \norm{\grad \sigma(\vec{y})} \leq \epsilon \right\}\). Since \(|| \grad \sigma(\vec{r}_c(0))|| > \epsilon\), then \( \text{dist}(\vec{r}_c(0), A) = d > 0 \). 
    
    In particular, since \(||\dot{\vec{r}}_c|| \leq u_r\)  for \(t^* = \frac{d}{2u_r}\), it follows that \(|| \grad \sigma(\vec{r}_c(t))|| > \epsilon\) for all \(t \in [0, 2t^*]\).

    Therefore, in the interval \([0, t^*]\), we are under the conditions of Lemma \ref{lema-cotaomega}. That is, there exists \(\kappa_1\) such that for \(k_\gamma \geq \kappa_1\), we have \(|\omega_d| \leq \Omega_d\). In particular, for all \(k_\gamma \geq \kappa = \max\left\{\kappa_1, \frac{2\Omega_d}{\gamma}, \frac{2}{t^*}\ln{\left(\frac{\pi}{\gamma}\right)}\right\}\), it follows that if \(\delta_{i,*} \geq \gamma\), then
    \begin{equation*}
        \dot{\delta}_{i,*} \leq -k_{\gamma} \delta_{i,*} + \Omega_d \leq -k_\gamma \delta_{i,*} + k_\gamma \frac{\delta_{i,*}}{2} =  -k_\gamma \frac{\delta_{i,*}}{2}.
    \end{equation*}
    Conversely, if \(\delta_{i,*} \leq - \gamma\), we have
    \begin{equation*}
        \dot{\delta}_{i,*} \geq - k_\gamma \delta_{i,*} - \Omega_d \geq - k_\gamma \frac{\delta_{i,*}}{2}.
    \end{equation*}
    This implies that \((- \gamma, \gamma)\) is an attractive interval. Specifically, once within this interval, and provided the gradient bound is maintained, the uniqueness of solutions guarantees that it is impossible to exit it.

    Therefore, we only need to show that we enter this interval at some \(t < t^*\). If \(|\delta_{i,*}(0)| \leq \gamma\), we have already shown that we remain within the interval. Otherwise, assume \(\delta_{i,*}(0) \geq \gamma\); the case \(\delta_{i,*}(0) \leq -\gamma\) follows similarly. 

    While \(\delta_{i,*}(t)\) is not within the interval \((- \gamma, \gamma)\), we have
    \begin{equation*}
        \dot{\delta}_{i,*} \leq -k_\gamma \delta_{i,*}/2 \implies \delta_{i,*}(t) \leq \delta_{i,*}(0)e^{-k_\gamma t/2}.
    \end{equation*}
    For \(t = t^*\) the right-hand side of the inequality is strictly less than \(\gamma\) and thus, at some point, we must have entered the interval. Since the interval is attractive, we will remain within it at time \(t^*\), as required. The fact that the bound on \(\delta_{i,*}\) continues to hold while \(|| \grad \sigma(\vec{r}_c(t))|| > \epsilon\) is automatic.
\end{proof}

Finally, it is important to note that these technical difficulties have minimal relevance in practical applications and are primarily necessary for formal mathematical treatment. Specifically, conditions such as the non-degeneracy of the geometry \(\vec{X}\), while crucial for theoretical development, are virtually negligible in practical applications. Similarly, bounds on \(\omega_d\) can be treated as assumptions of favourable behaviour regarding the \(\vec{m}_d\) field if deemed appropriate. In summary, these technical lemmas justify the intuitive notion that the robots' angular velocity is sufficiently rapid to effectively follow changes in the field and that once the robots align with the field, their geometry will remain nearly rigid.

Using the results previously established, we are now prepared to provide proof of convergence for the system described in \eqref{eq:augmentado}, at least up to regions where the gradient is small. Here, the smallness of the gradient is controlled by the value of \(\epsilon\) that we select.

\begin{proof}[Proof of Theorem \ref{th: main-unic}]
We will show that there exists a value \( k_\gamma \) such that the centroid of the trajectories is confined within the open set
\begin{equation*}
    B = \left\{\vec{y} \in \mathbb{R}^2: \|\vec{r}^* - \vec{y}\| < \epsilon \right\}.
\end{equation*}
As before, we assume without loss of generality that \(\vec{x}_1(0)\) and \(\vec{x}_2(0)\) form a basis. Reasoning as in Lemma \ref{lema-lower_bound}, we can choose \(\kappa_1\) sufficiently large so that \(\vec{x}_1(t)\) and \(\vec{x}_2(t)\) always form a basis of the space and \(\vec{X}\) remains within a compact set.

We first wish to study the angle between our guiding field \(\vec{m}_d\) and \(\nabla \sigma (\vec{r}_c)\). This can be examined by considering the angle between \(\frac{\vec{L}_\sigma(\vec{R})}{\|\nabla \sigma (\vec{r}_c)\|}\) and \(\frac{\nabla \sigma (\vec{r}_c)}{\|\nabla \sigma (\vec{r}_c)\|}\), which we denote by \(\theta(t)\). Taking the scalar product gives
\begin{align*}
     H\left(\vec{X}, \frac{\nabla \sigma (\vec{r}_c)}{\|\nabla \sigma (\vec{r}_c)\|}\right) &:= \frac{\vec{L}_\sigma(\vec{R}(t)) \cdot \nabla \sigma(\vec{r}_c)}{\|\nabla \sigma(\vec{r}_c(t))\|^2} \\
     &=
         \frac{1}{ND^2}\sum_{i=1}^N \left| \frac{\nabla \sigma(\vec{r}_c)}{\|\nabla \sigma(\vec{r}_c)\|} \cdot \vec{x}_i \right|^2 > 0,
\end{align*}
so \(H\) is a continuous function in its parameters. The first parameter \(\vec{X}\) is contained in a compact set \(F\), while the second parameter is in the unit sphere \(\mathbb{S}\), so the function \(H\) achieves a minimum value \(H^* > 0\) in the compact set \(F \times \mathbb{S}\). Thus,
\begin{equation*}
    0 < H^* \leq \frac{\vec{L}_\sigma(\vec{R}) \cdot \nabla \sigma(\vec{r}_c)}{\|\nabla \sigma(\vec{r}_c)\|^2} = \cos{(\theta)} \frac{\|\vec{L}_\sigma(\vec{R})\|}{\|\nabla \sigma(\vec{r}_c)\|} \leq  \cos{(\theta)},
\end{equation*}
where in the last inequality we have used that
\begin{align*}
    \|\vec{L}_\sigma(\vec{R})\| &= \frac{1}{ND^2} \|\sum_{i=1}^N (\nabla \sigma(\vec{r}_c) \cdot \vec{x}_i) \vec{x}_i\|\\
    &\leq \frac{\|\nabla \sigma(\vec{r}_c)\| }{N} \sum_{i=1}^N \frac{\|\vec{x}_i\|^2}{D^2} \leq \|\nabla \sigma(\vec{r}_c)\|.
\end{align*}
Therefore, we quickly obtain
\begin{equation*}
    0 < H^* \leq  \cos{(\theta)},
\end{equation*}
so that \( |\theta| \leq \arccos{(H^*)} < \frac{\pi}{2} \). In particular, there exists \(\gamma > 0\) such that \( |\theta(t) + \gamma| < \frac{\pi}{2} \) when moving throughout the trajectory.

Let \(\epsilon^*\) be small enough so that
\begin{equation*}
    A = \left\{\vec{y} \in B: \|\nabla \sigma(\vec{y})\| \leq \epsilon^* \right\} \subset  B
\end{equation*}
satisfies
\begin{equation*}
    d = \text{dist}\left(A, \mathbb{R}^2 \backslash B \right) > 0.
\end{equation*}
Applying Lemma \ref{lema-gamma} for this \(\epsilon^*\) and \(\gamma\), there exists \(\kappa_2\) such that if \(k_\gamma \geq \kappa_2\), then there exists \(t^*\) such that \( |\delta_{i,*}(t)| < \gamma \) for all \( i = 1, \dots, N \) from this time \(t^*\) until \(\|\nabla \sigma(\vec{r}_c(t))\| \leq \epsilon^*\). In particular, while the conditions hold, the function \(\sigma(\vec{r}_c(t))\) is increasing since
\begin{align*}
    \frac{d}{dt} \sigma(\vec{r}_c(t)) &= \frac{d \vec{r}_c(t)}{dt} \cdot \nabla \sigma(\vec{r}_c(t)) \\
    &= \frac{u_r}{N} \sum_{i=1}^N \vec{m}_i(\alpha_i(t)) \cdot \nabla \sigma(\vec{r}_c(t)) > 0,
\end{align*}
where we have used that \(\vec{m}_i(t) \cdot \nabla \sigma(\vec{r}_c(t)) > 0\) by the previously proven angular relation. Reasoning as in Theorem \ref{th:convergencia-1}, the magnitude of the field in the trajectory will increase until it enters the set \(A\). Once inside, we can no longer control \(\sigma(\vec{r}_c(t))\), but we continue to know that \(\vec{X}\) is stable. Thus, if the trajectory escapes from \(A\), reasoning as in Lemma \ref{lema-gamma}, if we take \(k_\gamma \geq \kappa_3 = \max{\left\{\kappa_4, \frac{2\Omega_d}{\gamma}, \frac{2u_r}{d} \ln{\left(\frac{\pi}{\gamma} \right)}\right\}}\) where \(\kappa_4\) is given by\footnote{Technically, one would need to be careful to ensure that the angle is within $(-2\pi, 2\pi)$, this does not constitute a problem, only notational care, since we can always normalize the angles if they fall outside of this interval.} Lemma \ref{lema-cotaomega} with \(\|\nabla (\sigma(\vec{r}_c(t)))\| \geq \epsilon^*\), then \( |\delta_{i,*}(t)| < \gamma \) for all \( i = 1, \dots, N \) before it can escape \(B\), so it will return close to the maximum until falling into \(A\) repeatedly. Therefore, for \(k_\gamma \geq \kappa = \max{\{\kappa_1, \kappa_2, \kappa_3\}}\), the trajectory remains trapped in the set \(B\). In particular, there exists a time \(t_0\) such that for all \(t \geq t_0\), \(\|\vec{r}^* - \vec{r}_c(t)\| < \epsilon\), as desired.
\end{proof}

Lastly, note that in contrast to Theorem \ref{th:convergencia-1} where the dynamics were valid for any \(\epsilon\), in this case, the dynamics depend on the chosen \(\epsilon\) in the search problem. This is natural. Since robots never stop moving they end up rotating around the maximum, while the ratio \(u_r/k_\gamma \) determines how large the region over which the swarm oscillates will be.

\end{document}